\documentclass[letterpaper, 10 pt, conference]{ieeeconf}  % Comment this line out if you need a4paper

\IEEEoverridecommandlockouts                              % This command is only needed if 
\usepackage{amsmath,amsthm,amssymb}
\usepackage{graphicx,fixltx2e}
\usepackage{hyperref}
\usepackage{xy}
\usepackage{bbm}
\usepackage{todonotes}
\usepackage{float}
\usepackage{cite}
\usepackage{algorithm}
\usepackage[noend]{algpseudocode}
\usepackage{subcaption}
\usepackage{multicol}
\makeatletter
\def\algbackskip{\hskip-\ALG@thistlm}
\makeatother

\newcommand{\N}{\mathbb{N}}

\newcommand{\tbx}{\textbf{x}}

\theoremstyle{remark}

\newtheorem{problem}{Problem}
\newtheorem{lemma}{Lemma}

\newtheorem{assumption}{Assumption}

\newtheorem{theorem}{Theorem}
\newtheorem{corollary}{Corollary}
\newtheorem{proposition}{Proposition}

\makeatletter
\newcommand{\pushright}[1]{\ifmeasuring@#1\else\omit\hfill$\displaystyle#1$\fi\ignorespaces}
\newcommand{\pushleft}[1]{\ifmeasuring@#1\else\omit$\displaystyle#1$\hfill\fi\ignorespaces}
\makeatother

\allowdisplaybreaks

\title{\LARGE \bf
Local Trajectory Stabilization for Dexterous Manipulation via Piecewise Affine Approximations
}

\author{Weiqiao Han and Russ Tedrake%$^{1}$% <-this % stops a space
\thanks{Computer Science and Artificial Intelligence Laboratory, Massachusetts Institute of Technology, 77 Massachusetts Avenue, Cambridge, MA 02139, USA. {\tt\small weiqiaoh,russt@mit.edu}
}%
\thanks{This work was supported by Air Force/Lincoln Laboratory Award No. PO\# 7000374874, and Lockheed Martin Corporation Award No. RPP2016-002.}
}

\begin{document}

\newcounter{eqn}

\maketitle
\thispagestyle{empty}
\pagestyle{empty}

%%%%%%%%%%%%%%%%%%%%%%%%%%%%%%%%%%%%%%%%%%%%%%%%%%%%%%%%%%%%%%%%%%%%%%%%%%%%%%%%
\begin{abstract}
We propose a model-based approach to design feedback policies for dexterous robotic manipulation. The manipulation problem is formulated as reaching the target region from an initial state for some non-smooth nonlinear system. First, we use trajectory optimization to find a feasible trajectory. Next, we characterize the local multi-contact dynamics around the trajectory as a piecewise affine system, and build a funnel around the linearization of the nominal trajectory using polytopes. We prove that the feedback controller at the vicinity of the linearization is guaranteed to drive the nonlinear system to the target region. During online execution, we solve linear programs to track the system trajectory. We validate the algorithm on hardware, showing that even under large external disturbances, the controller is able to accomplish the task.  
\end{abstract}

%%%%%%%%%%%%%%%%%%%%%%%%%%%%%%%%%%%%%%%%%%%%%%%%%%%%%%%%%%%%%%%%%%%%%%%%%%%%%%%%
\section{Introduction}
How to enable robots to manipulate objects dexterously like human hands do is a long-standing problem \cite{mason2018toward}. 
Alongside hardware, perception, and planning, motor control is one of the challenges in manipulation.
Designing reliable feedback controllers for manipulation is hard, due to the nonlinear and contact-rich nature of the manipulation tasks.
For example, how should we design a controller for the robot to flip a carrot (half-cylinder) with the flat surface facing upwards (Fig. \ref{fig:carrot}) to the pose where the flat surface is facing downwards?
This is more than a simple pick-and-place task and indeed even an experienced human operator tele-operating the robot often cannot succeed in one or two tries (see the accompanying video).

In general, there are two main categories of approaches to control design for manipulation -- model-based approaches and learning-based approaches. Model-based approaches have mainly been applied to grasping \cite{li1988task,miller2003automatic}, and planar pushing \cite{chavan-dafle2018rss,zhou2016convex}. 
They are usually specific to the hardware. Most model-based controllers are open loop \cite{dafle2014extrinsic}. 
In order to achieve the dexterity of a human hand, we want the robot to do tasks more complicated than grasping and planar pushing in a feedback fashion.

On the other end of the spectrum, learning-based approaches have been applied to tasks with greater variety and difficulty, ranging from moving the end-effector to a target pose \cite{levine2015learning}, grasping \cite{levine2018learning}, and planar pushing \cite{finn2017deep},  to rotating a long rod \cite{kumar2016optimal}, throwing objects \cite{zeng2019tossingbot}, and rotating a cube \cite{andrychowicz2018learning}. 
Though learning-based approaches have been successful on tasks that model-based approaches have not been able to solve, they lack reliability or stability guarantees.
Motions of manipulators are unpredictable, especially for those whose policies are represented by deep neural networks. 
Judging the usefulness of such control policies is always empirical.

\begin{figure}[t]
\centering
\includegraphics[width=0.4\textwidth,angle=-90]{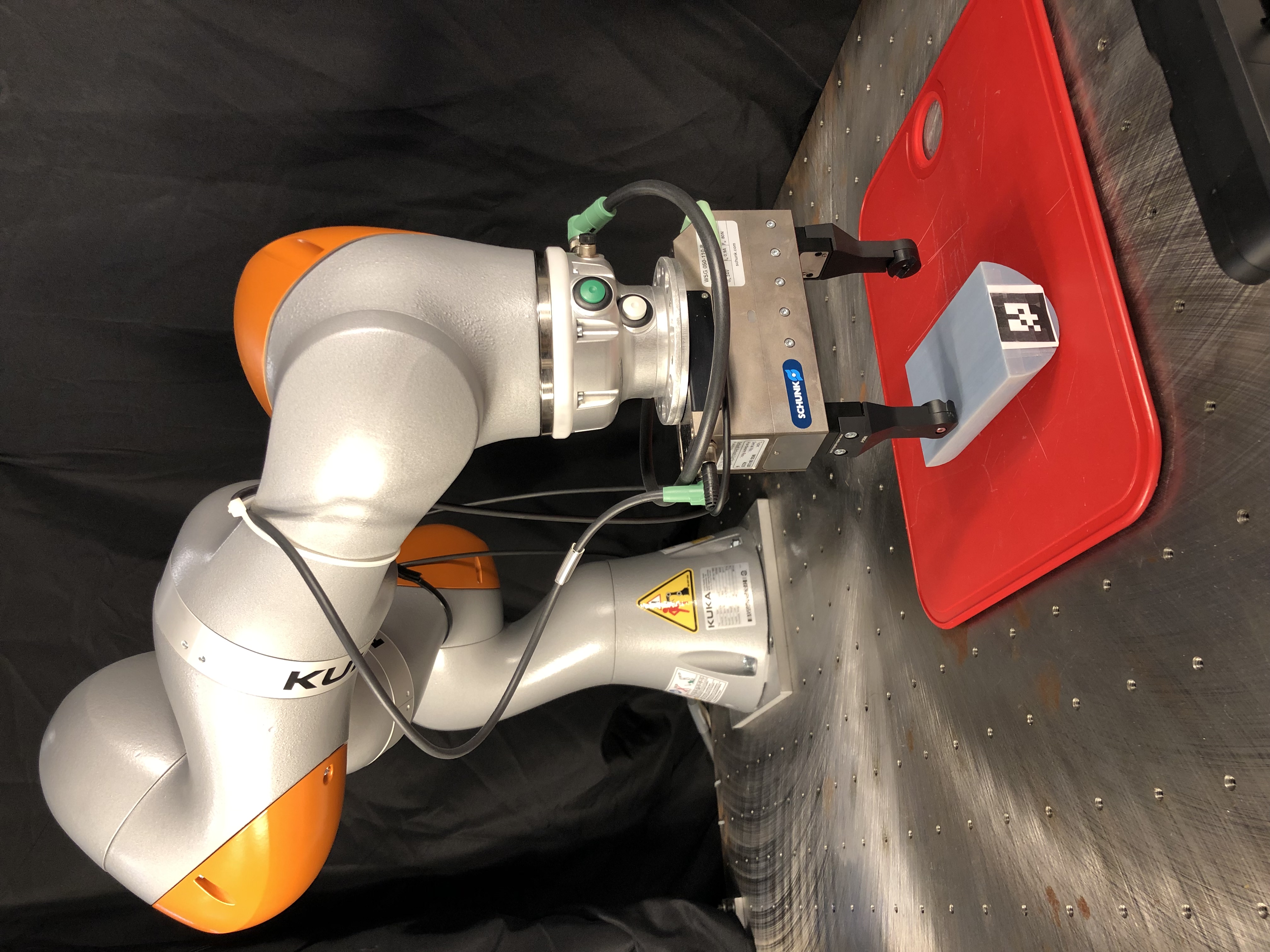}
\includegraphics[width=0.2\textwidth]{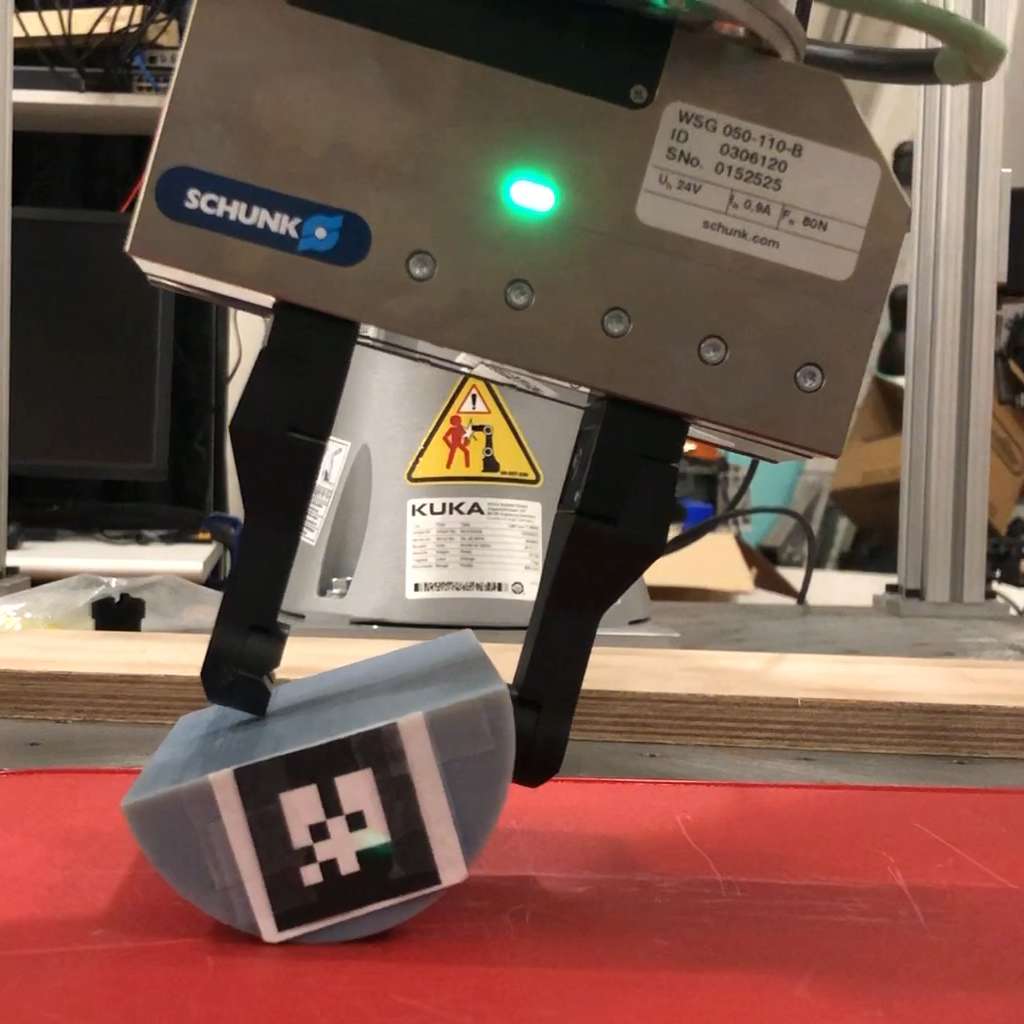}
\includegraphics[width=0.2\textwidth]{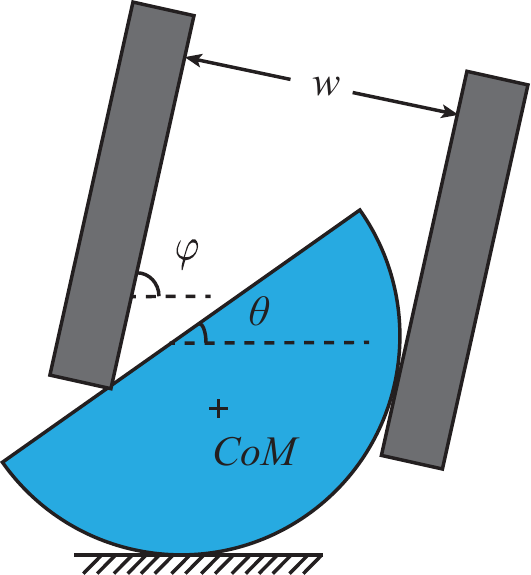}
\caption{Flipping a carrot (half-cylinder) using a parallel gripper}
\label{fig:carrot}
\end{figure}

Our approach falls into the first category.
We present a general algorithm of feedback design for dexterous manipulation.
We draw ideas from recent advances in humanoid robots path planning \cite{kuindersma2016optimization} and robust control synthesis \cite{majumdar2017funnel,sadraddini2019sampling}.
We formulate the manipulation problem as using a manipulator to change the object pose from its initial pose to a target pose.
Our algorithm consists of the following steps.
First, the pose of the object and that of the robot are incorporated into the system state, and a trajectory optimization method that has been used to design nominal trajectories for humanoid robots walking over uneven terrains is deployed to design a nominal trajectory for manipulating the object.
Second, piecewise-affine (PWA) linearization around the nominal trajectory is computed and a funnel around the nominal mode of the PWA linearization is formed.
Third, a linear program (LP) is solved at each time step during online execution, achieving real time feedback control. 
The algorithm can be implemented on common robotic platforms.
It can be applied to more complicated tasks than grasping and planar pushing, and in particular, it can flip the half-cylinder.
The controller is robust to moderate external disturbances.
Our contributions are (1) designing a feedback control algorithm for dexterous manipulation, bridging the gap between locomotion/UAV control and robotic manipulation; (2) providing robustness guarantees for our algorithm (Proposition \ref{prop1}); (3) validating the algorithm on hardware.

\section{Related Work}
\textit{Model-based feedback control for manipulation}:
Lynch's group used feedback control for sliding \cite{shi2017dynamic}, rolling \cite{ryu2013control}, vibratory manipulation \cite{umbanhowar2012effect,vose2012manipulation}, and hybrid manipulation using motion primitives \cite{woodruff2017planning,dafle2014extrinsic}.
They designed specific manipulators for specific tasks.
In contrast, our algorithm is more general and can be carried out on common robot platforms.
Rodriguez's group used feedback control for planar pushing \cite{hogan2018reactive}.
They linearize the nonlinear system and solve linear model predictive control (MPC) online to plan the trajectory.

\textit{Path planning}:
There are generally two categories of approaches to path planning.
One is motion planning algorithms \cite{lavalle2006planning}. 
For discretized configuration space, grid search algorithms such as A$^*$ and its variants \cite{likhachev2005anytime} are widely used.
Sampling-based algorithms such as rapidly exploring random trees \cite{shkolnik2010sample}  are common approaches for continuous configuration space. 
The other category is the trajectory optimization algorithms.
A common approach in this category would be to formulate the problem as nonlinear optimization programs and solve it using off-the-shelf numerical solvers \cite{posa2014direct,dai2014whole,mordatch2012discovery}.
Other approaches in this category include augmented Lagrangian \cite{toussaint2014novel}, mixed-integer convex optimization (MICP) \cite{valenzuela2016mixed}, differential dynamic programming (DDP) \cite{tassa2014control}, and iterative linear quadratic Gaussian (iLQG) \cite{tassa2012synthesis}. 
A combination of these two methods have proved even more useful \cite{dolgov2010path,zhang2018autonomous}. 
In this work, we use trajectory optimization and use off-the-shelf numerical solvers to solve a nonlinear optimization program.
% We follow a similar methodology in combining sampling-based and optimization-based trajectory synthesis.    

\textit{Local feedback controllers}: In control literature, it is quite standard to track a system trajectory using linear quadratic servo (LQ servo) or time-varying linear quadratic regulator (TVLQR) based on linearization of the system trajectory around nominal states \cite{anderson2007optimal}.
In reinforcement learning, it is common to learn local linear models of the system and linear feedback control gains \cite{levine2015learning,kumar2016optimal}.
However, the local linear model does not fully capture the contact-rich nature of manipulation. 
Robot fingers may make and break multiple contacts with the object due to small disturbances. 
When contact modes change, the dynamics may change.
So it is more natural to model the local dynamics as a PWA system, i.e., a system whose state-input space is partitioned into several polytopic regions, with each region
associated with a different affine dynamics equation.

However, stabilizing PWA systems alone is not an easy problem. 
Explicit solutions of optimal control for PWA systems can be computed offline by multi-parametric  programming \cite{baoti2006constrained, christophersen2005optimal,  borrelli2005dynamic, baric2008efficient, bemporad2000optimal,bemporad2000piecewise,bemporad2002optimal,mayne2002optimal}.
The computational complexity of these methods grows exponentially with respect to the number of time steps.
Lyapunov-based approaches \cite{rodrigues2002dynamic,lazar2006model,han2017feedback} and occupation measure approaches \cite{zhao2017optimal,han2019controller} do not depend on the number of time steps, but are quite conservative and may not always find solutions.
Sampling-based methods \cite{marcucci2017approximate,sadraddini2019sampling} suffer from the issue of scalability.
In this work, we consider manipulation problems in which there is only one rigid object, the manipulators are fully actuated, and the PWA dynamics is caused by manipulators making and breaking contacts with the object.
We track the system trajectory by solving LP online.

% \subsection*{Notation}

\section{Problem Statement and Approach}
In many manipulation problems, the goal is for the manipulator to change the pose of an object from the initial pose to some target pose(s) specified by the user, for example, using a parallel gripper to flip a half-cylinder from the pose with the flat surface facing upwards to the pose with the flat surface facing downwards.
We incorporate the pose of the manipulator and the pose of the object into the system state $x$.
Then the manipulation problem is turned into a control problem:
Design a feedback controller $u(x)$ that drives the initial state $x_0$ to a target region $X_G$.
The dynamics of the manipulator-object system, $\dot{x} = f(x,u)$, is nonlinear and non-smooth.

Our algorithm contains both offline planning and online execution phases.
During the offline planning phase, we formulate a nonlinear trajectory optimization problem that drives $x_0$ to a target state $x_N \in X_G$ in $N$ time steps.
We solve the trajectory optimization using off-the-shelf numerical solvers. 
The solution is a nominal trajectory $\{\bar{x}_0,\bar{u}_0,\ldots,\bar{x}_{N-1},\bar{u}_{N-1}, \bar{x}_N\}$.
This is an open-loop trajectory and might be fragile under external disturbances.
We build funnels around (the linearization of) the nominal trajectory using polytopes.
If the system state falls into the funnel, the system is guaranteed to reach the target region.
In order to resist larger disturbances, we compute the PWA linearization around the nominal trajectory when the manipulator makes and breaks contacts with the object.
As mentioned before, we assume that there is only one rigid object, that the manipulators are fully actuated, and that the local PWA dynamics is only caused by manipulators making and breaking contacts with the object.
During online execution, we solve linear programs to steer the system into the polytopes or onto the nominal points.

In summary, our algorithm consists of the following three steps: (1) solving nonlinear trajectory optimization to find the nominal trajectory offline; (2) computing the PWA linearization and building a polytopic funnel around the linearization of the nominal trajectory offline; (3) solving LP's online to drive the system around the nominal points.

\section{Trajectory Planning}
\label{sec:trajectory}
\subsection{Trajectory Optimization}
We plan the path using trajectory optimization methods.
In particular, we use direct transcription as in \cite{dai2014whole}.
The continuous-time dynamics $\dot{x} = f(x,u)$ is discretized into a discrete-time system $x^+ = \phi(x,u)$ with sampling time $dt$.
The trajectory is discretized into $N$ time steps with $N\cdot dt = T$, where $T$ is the time horizon:
\begin{align*}
    \text{minimize}\ & \sum_{t = 0}^T L(x[t],u[t])\\
    \text{subject to} \ & m \ddot{{r}}[t] = m{g} + \sum_j {F}_j[t] \\
    & I \ddot{\theta}[t] = \sum_j ({c}_j[k]-{r}[k])\times {F}_j[t] \\
    & \text{friction cone constraints}, \text{contact constraints}, \\
    & \text{kinematics constraints}, \text{time integration constraints}
\end{align*}
where ${r}[t]$ is the position of the center of mass, ${F}_j[t]$ are forces, ${c}_j[t]$ is the the contact position of the $j$-th force for each $j$, $L$ is the loss function, and the decision variables include states $x[t]$, controls $u[t]$, and variables in the contact constraints.
The variables ${r}[t]$ and ${c}_j[t]$ are part of the states $x[t]$, and the forces ${F}_j[t]$ are part of the controls $u[t]$.

The first two constraints are Newton-Euler equations.
The friction cone constraints for planar systems are
    $-\mu F_{j,n} \leq F_{j,t} \leq \mu F_{j,n}$,
where $F_{j,n}$ is the normal force and $F_{j,t}$ is the frictional force.
For 3-dimensional systems, we can use a polyhedral cone to approximate the friction cone \cite{stewart2000implicit}:
    ${F}_j[t] = \sum_{i} \beta_{ij} {w}_{ij}, \beta_{ij} \geq 0$,
where ${w}_{ij}$'s are the spanning vectors of the polyhedral cone.
For some contacts, we formulate the contact constraints as linear complementarity problems (LCP)  \cite{posa2014direct} to fully characterize all possible contact modes -- sticking, sliding, or breaking contacts.
We use IPOPT \cite{wachter2006implementation} to solve the trajectory optimization offline.
% \iffalse 
% \begin{align*}
%   &(\mu F_n - F_t^+ - F_t^-) \gamma = 0 \\
%   &(\gamma + v) F_1^+ = 0\\
%   &(\gamma - v) F_1^- = 0\\
%   & \mu F_n - F_1^+ - F_1^- \geq 0\\
%   & \gamma + v \geq 0\\
%   & \gamma - v \geq 0\\
%   & \gamma \geq 0\\
%   & F_t^+ \geq 0\\
%   & F_t^- \geq 0 
% \end{align*}
% where $v$ is the velocity of the contact point moving on the surface, and $\gamma$ is some slack variable.
% The friction is equal to $F_t^+-F_t^-$.
% There are three possible solutions to the LCP: (i) $v = 0$, and $\mu F_n - F_t^+ - F_t^- \geq 0$, i.e., the contact point is not moving and the friction cone constraint is satisfied; (ii) $v > 0$, $F_1^+ = 0$, and $\mu F_n = F_t^-$, i.e., the contact point is moving in the positive direction, and the friction is in the negative direction; and (iii) $v<0$, $F_1^- = 0$, and $\mu F_n = F_t^+$, in which the contact point is moving in the negative direction, and the friction is in the positive direction.
% \fi 
Since IPOPT is an interior point method solver, the LCP constraints $P(x)^\top Q(x) = 0, P(x) \geq 0,Q(x) \geq 0$, can be replaced by the equivalent constraints $P(x)^\top Q(x) \leq 0, P(x) \geq 0, Q(x) \geq 0$, and further be relaxed as $P(x)^\top Q(x) \leq \epsilon, P(x) \geq 0, Q(x) \geq 0$ for small $\epsilon > 0$.

% For the time integration constraints, we use forward Euler integration, so
% \begin{align*}
%     {r}[t+1] = {r}[t] + \dot{{r}}[t]dt
% \end{align*}
% and so on.

\subsection{Force as Control Input}
In the trajectory planning described in the previous subsection, we borrowed the idea of zero-moment point (ZMP) for bipedal footstep planning in the humanoid robot literature \cite{kajita2014introduction, kajita2003biped}. 
For a bipedal robot walking on the ground, ZMP is by definition a point on the ground where the sum of all the tangential moments equals zero.
Although many humanoid robots have pressure sensors on the feet, because of lack of reliability (in the case of Atlas), researchers do not measure ZMP directly. What they do is to plan a joint ZMP and center of mass (CoM) trajectory, and then only track the CoM trajectory during online execution \cite{kuindersma2016optimization}.

Similarly, in the trajectory optimization formulation, we use forces as part of the control input. 
In the manipulation context, the forces are those between the gripper and the object, and those between the object and the environment, e.g. the table.
Although the hardware we are using cannot directly measure the force it applies to the object, we still use forces as control variables to help plan the CoM trajectory of the object as well as the pose trajectory of the gripper.
During online execution, we only track the trajectory of the CoM of the object and that of the gripper.
We find in practice this approach works well.

\section{Local Feedback Control}
\subsection{Local Multi-Contact Dynamics}
\label{sec:pwa}
\begin{figure}[t]
  %\centering 
  \begin{minipage}{0.49\linewidth}
  \centering 
  \includegraphics[width=0.7\linewidth]{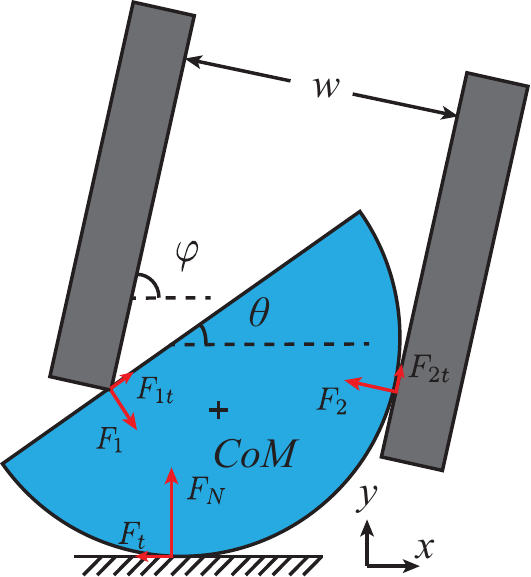}\\ 
  \small (A)
  \end{minipage}\ 
  \begin{minipage}{0.49\linewidth}
  \centering 
  \includegraphics[width=0.7\linewidth]{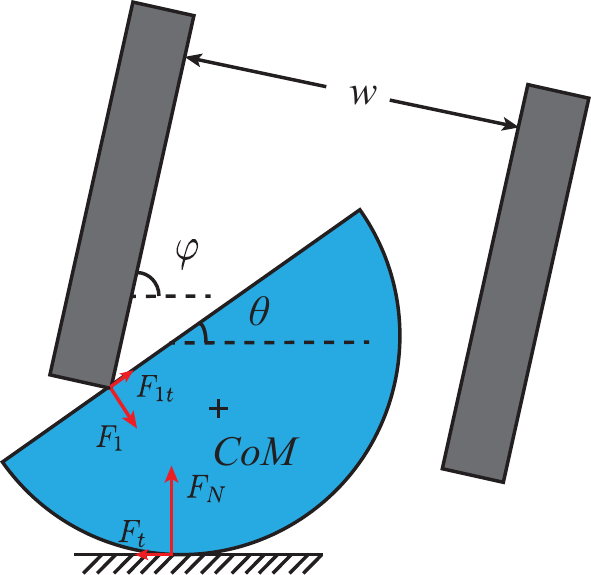} \\
  \small (B)
  \end{minipage}
  \caption{Two contact modes. (A) Right finger in contact with carrot. (B) Right finger not in contact with carrot.}
  \label{fig:contact_modes}
\end{figure}

%\subsection{Linearization around Nominal Trajectory}
From the nonlinear trajectory optimization, we obtain a nominal trajectory $\{\bar{x}_0,\bar{u}_0,\bar{x}_1,\bar{u}_1,\ldots,\bar{x}_N\}$, where $\bar{x}_N \in X_G$.
At each nominal point $(\bar{x}_i,\bar{u}_i)$, where $i = 0,\ldots, N$, we linearize the dynamics $\dot{x} = {f}(x,u)$ as $\dot{x} =  {A}_i(x-\bar{x}_i) + {B}_i (u-\bar{u}_i) + {c}_i$, where ${A}_i = {\partial {f} \over \partial x}(\bar{x}_i,\bar{u}_i), {B}_i = {\partial {f} \over \partial u}(\bar{x}_i,\bar{u}_i)$, and ${c}_i = {f}(\bar{x}_i, \bar{u}_i)$.
This continuous time affine system can be discretized as $x^+ =  \tilde{{A}}_i x + \tilde{{B}}_i u + \tilde{{c}}_i$.
This can equivalently be obtained by linearizing the discretized system $x^+ = \phi(x,u)$.
The corresponding state space $X_{i}^1$ is obtained by linearizing all constraints ${g}(x,u) \leq {0}$ at $(\bar{x}_i,\bar{u}_i)$.

Now we consider different contact modes due to making or breaking contacts between an object and the manipulator.
We fix a nominal point $(\bar{x}_i,\bar{u}_i)$.
Suppose there are $p\in \N$ contact locations that may make or break contacts.
Each contact mode corresponds to a distinct dynamics $\dot{x} = {f}_j (x,u)$ and constraints ${g}_j (x,u) \leq {0}$, where $j = 1,\ldots,s := 2^p$, ${f}_1 = {f}$, and ${g}_1 = {g}$.
For the half-cylinder example, the right finger can be touching or not touching the half-cylinder, giving two contact modes with distinct dynamics and distinct state space regions (Figure \ref{fig:contact_modes}).
We call the contact mode in which the nominal trajectory is computed the \textit{nominal mode} or Mode 1.
We linearize the dynamics and the constraints for modes other than the nominal mode and evaluate at $(\bar{x}_i,\bar{u}_i)$.
Since making and breaking contacts can happen when the system state makes very small changes, the linearization is in the vicinity of the nominal point and hence is valid.
In fact, the nominal points can be on the boundaries of state space cells of a few modes.
Thus we obtain a piecewise affine system $x^+ =  {{A}}_{i,j}x + {{B}}_{i,j} u + {{c}}_{i,j} = : {h}_{i,j}(x,u)$ with state space $X_i^j$, $j = 1,\ldots,s$, around each nominal point $(\bar{x}_i,\bar{u}_i)$, $i = 1,\ldots,N$.
We call $x^+ = {h}_{i,1}(x,u)$ the \textit{nominal linearization}.
%We are going to stabilize this local ``time-varying'' piecewise affine system.

\label{sec:control}

\subsection{Polytopic Funnel around Nominal Trajectory}
After we get a nominal trajectory, we are going to build a funnel around the trajectory so that if the system state is inside the funnel, it will always stay inside the funnel until reaching the target region.
Funnels can be sum-of-squares (SOS) \cite{tedrake2010lqr,majumdar2017funnel} or polytopic  \cite{sadraddini2019sampling}.
We use the latter, because numerical computations involving polytopes requires solving LP or quadratic program (QP), which are more reliable than solving SOS programs.

Here we briefly review the polytopic tree method in \cite{sadraddini2019sampling}. 
Suppose the system is a time-varying affine system
\[
x_{t+1} = A_t x_t + B_t u_t + c_t.
\]
Given a target region $X_G$ and time horizon $N$, the method computes a trajectory $\{\bar{x}_i,\bar{u}_i\}_{i=0}^N$ alongside with polytopes 
\begin{align}
    Y_i = \{\bar{x}_i\} \oplus G_i \mathbb{P} \label{poly_Y}
\end{align}
and a control law 
\begin{align}
    u_i(x) = \bar{u}_i + \theta_i p(x) \label{poly_u}
\end{align}
around each point $(\bar{x}_i,\bar{u}_i)$ on the trajectory by solving a main LP.
Here $\oplus$ represents Minkowski sum, $G_i$ and $\theta_i$ are decision matrices the main LP searches over, $G_{i+1} = A_iG_i+B_i\theta_i$ captures the evolution of the polytopes over time with respect to the system dynamics, $\mathbb{P}$ is the hyper-cube $[-1,1]^n$, and $p(x) \in \mathbb{P}$ satisfies 
\begin{align}
    x = \bar{x}_i + G_i p(x). \label{poly_x}
\end{align}
The main LP to be solved offline encodes the state space constraints in polytopic containment forms, including the final polytopic containment constraint $Y_N \subseteq X_G$, as well as trying to maximize the volumes of the polytopes.
%Note that each $Y_i$ is an affine map of a hyper-cube, and hence is a zonotope \cite{ziegler2012lectures}.   
% Note that the trajectory is computed by solving MICP instead of nonlinear trajectory optimization, because the method assumes the entire system is PWA. 
During online execution, if the current state $x$ is in some polytope $Y_i$, then $p(x)$ can be found by solving an LP through Equation (\ref{poly_x}) and $u_i(x)$ can be calculated by Equation (\ref{poly_u}). 
By following the control law $u_i(x)$, the system is guaranteed to land inside the next polytope $Y_{i+1}$ and hence eventually it will reach $Y_N\subseteq X_G$.

We use the method to build polytopes for the nominal linearization $x^+ = h_{i,1}(x,u)$ around the nominal trajectory $(\bar{x}_i,\bar{u}_i)$.
While \cite{sadraddini2019sampling} deals with PWA systems, we show that the polytopic tree method can be extended to non-smooth nonlinear systems.
\begin{proposition}\label{prop1}
If $x^+ = \phi(x,u)$ and $\phi$ is Lipschitz continuous, and if the polytopic tree method finds the polytopes $Y_i$ as in Equation (\ref{poly_Y}) for the nominal linearization $x^+ = h_{i,1}(x,u)$ at the nominal trajectory, with $\mathbb{P} = [-1,1]^n$ and $G_i$ full rank $\forall i$, then there exists $\mathbb{P}_i = [-a_i,a_i]^n, 0 < a_0 \leq a_1 \leq \cdots \leq a_N = 1$ such that if $x \in \tilde{Y}_i :=  \{\bar{x}_i\} \oplus G_i \mathbb{P}_i$, then by following $u$ as in Equation (\ref{poly_u}), $\phi(x,u) \in \tilde{Y}_{i+1}$. So $x$ will eventually land in $\tilde{Y}_N = Y_N \subseteq X_G$.
\end{proposition}
\begin{proof}[Proof Sketch]
If the system state $x \in \tilde{Y}_{N-1} =  \{\bar{x}_{N-1}\} \oplus G_{N-1} \mathbb{P}_{N-1}$, by following $u_{N-1}$, $x^+ \in \{\bar{x}_{N} + e_{N-1}\} \oplus G_{N} \mathbb{P}_{N-1}$, where $e_{N-1}$ is the residual error induced by the linearization.
Since the dynamics $\phi$ is Lipschitz, $e_{N-1}$ goes to 0 as $(x,u)$ goes to $(\bar{x}_{N-1},\bar{u}_{N-1})$. 
Given small $\varepsilon > 0$, we can find $\delta > 0$ such that if $||(x,u)-(\bar{x}_{N-1},\bar{u}_{N-1})||_2 < \delta$, then $||e_{N-1}||_2 < \varepsilon$.
We can choose $\varepsilon$ and $a_{N-1}$ such that any $(x,u)$ satisfying $x \in \tilde{Y}_{N-1}$ and Equation (\ref{poly_u}) also satisfies $||(x,u)-(\bar{x}_{N-1},\bar{u}_{N-1})||_2 < \delta$ and such that $e_{N-1} \oplus G_N \mathbb{P}_{N-1} \subseteq G_N \mathbb{P}_{N}$.
Then $x^+ \in \tilde{Y}_N$.
The proof is complete by repeating the procedure for $i=N-2,\ldots,0$.
\end{proof}
The Proposition says that if the system state is close enough to the nominal trajectory and if the nonlinear dynamics is Lipschitz, then we can use the {affine} control law (\ref{poly_u}) for the nominal linearization of the trajectory to steer the nonlinear system to the target region for sure.
%This gives us certain guarantees locally.
{
This gives us stability guarantees near the vicinity of the nominal trajectory.
In practice, for any arbitrary system, we do not know how large its polytopic funnel can be, or how likely it is for the state to fall into the funnel.}
We want to be able to handle larger external disturbances during online execution.
This is what we are going to {address} in the next section.

\section{Online Execution}
The polytopic tree method in \cite{sadraddini2019sampling} hopes to probabilistically cover the state space by polytopes by growing a single polytopic trajectory to an existing polytopic tree, similar in methodology to the growth of the LQR-trees \cite{tedrake2010lqr}.
The method samples points in the state space, steers sample points to the current polytopic tree as well as building polytopes along the way by solving mixed-integer linear program (MILP), and enlarges the current tree by adding the new polytopes to the tree.

In practice, for example for the half-cylinder flipping experiment, there are several problems with the polytopic tree method.
First, the volumes of the polytopes {can be} very small, and hence a state {may} never fall into any polytope.
Second, the polytopic tree method requires checking the closest polytope online, which is computationally inefficient in the naive implementation when there are large number of polytopes.
Third, the polytopic tree method deals with PWA systems, but our system is nonlinear and computing trajectory from a sample point to the current tree is potentially an expensive nonlinear trajectory optimization problem which cannot be carried out online.

Therefore, we propose the practical improvement of the polytopic tree method, with the sacrifice of stability guarantees when there are large deviations to the nominal trajectory.
We only keep one nominal trajectory, which is computed in Section \ref{sec:trajectory}.
We build polytopes $Y_i = \{\bar{x}_i\} \oplus G_i \mathbb{P}, i=0,\ldots,N$ around the nominal linearization of the trajectory.
This amounts to solving an LP.
During online execution, we compute the closest nominal state $\bar{x}_i$ to the current state $x$ (with respect to some weighted $L_2$ norm) and determine the current contact mode $j$.
If $x$ is inside the polytope $\{\bar{x}_i\} \oplus G_i \mathbb{P}$, the we use the corresponding control law $u_i(x) = \bar{u}_i + \theta_i p(x)$, where $x = \bar{x}_i + G_i p(x)$. 
Otherwise we let the target index be $v = \min\{i+1,N\}$ and solve the following LP to get control $u$:
\begin{align}\label{lp1}
    &\quad \ \min_{\gamma,p,\delta,u} \ \alpha^\top \gamma \\
    &\text{subject to} \ x_{v} + G_{v} p = h_{i,j}(x,u) + \delta \nonumber\\
    & \quad \quad \quad \quad \ p \in \mathbb{P},  |\delta_k| \leq \gamma_k, k=1,\ldots,n  \nonumber
\end{align}
where $\alpha$ is some weight or cost vector.
This LP means we want the state to get to the polytope with index $v$ as close as possible.
When, for example in the half-cylinder flipping experiment, the volumes of the polytopes are very small, we can directly solve the LP
\begin{align}\label{lp2}
    &\quad \quad\min_{\gamma,\delta,u} \ \alpha^\top \gamma \\
    &\text{subject to} \ x_{v} = h_{i,j}(x,u) + \delta \nonumber\\
    & \quad \quad \quad \quad \  |\delta_k| \leq \gamma_k, k=1,\ldots,n  \nonumber
\end{align}
which means we want the state to get to the nominal state with index $v$ as close as possible.

\begin{algorithm}
\caption{Stabilizing controller around nominal trajectory}\label{alg:stabilizing_controller}
\hspace*{\algorithmicindent} \textbf{Input} Current state $x \not\in X_G$ \\
\hspace*{\algorithmicindent} \textbf{Output} Control $u$
\begin{algorithmic}[1] 
    \If {$x \in \{\bar{x}_i\} \oplus G_i \mathbb{P}, \forall i \in I$, for some index set $I$,} \Return $u = \bar{u}_{i_0} + \theta_{i_0} p(x)$, where $x = \bar{x}_{i_0} + G_{i_0} p(x)$, and ${i_0}$ is the largest element in $I$.
    \EndIf
    \State Find the closest nominal state $\bar{x}_i$ to $x$, w.r.t. some weighted $L_2$ norm.
    \State Determine the current contact mode $j$.
    \State Let the target index be $v = \min\{i+1,N\}$.
    \State Solve LP (\ref{lp1}) or (\ref{lp2}), \Return $u$.
\end{algorithmic}
\end{algorithm}

The procedure is summarized in Algorithm \ref{alg:stabilizing_controller}.
There can be many variants to Steps 4 and 5.
For example, one might use MPC-style planning based on local PWA linearization.
During offline phase, one samples states $\tilde{x}_k$ not in the nominal mode and solve MICP to get to some target points $\bar{x}_{v(k)}$ on the nominal trajectory, hence storing a list of samples $\{($state $\tilde{x}_k$, mode sequence to get to target $\bar{x}_{v(k)})\}_{k=1}^M$.
The target index $v(k)$ for each sample state $\tilde{x}_k$ can be chosen by comparing the cost to get to all nominal states $\bar{x}_i, i=0,\ldots,N$.
During online execution, one finds the closest sample $\tilde{x}_k$ to the current state $x$ and solve QP or LP to get to $\bar{x}_{v(k)}$ fixing the mode sequence as stored.

We find empirically that for the half-cylinder flipping experiment, solving LP like (\ref{lp1}) or (\ref{lp2}) to directly go to the nominal trajectory is more efficient than MPC-style planning which plans multiple steps to reach the nominal trajectory.
This might be because of our assumptions that the change of contact modes is only caused by the manipulator making and breaking contacts with the object and that the manipulator is fully actuated.
So instructing the manipulator to directly go back to the desired position works.
Also MPC-style planning on linearized PWA systems may accumulate linearization errors.

During the online execution, we use $\mathbb{P} = [-1,1]^n$ instead of $\mathbb{P}_i$ in Proposition \ref{prop1}, because it is more computationally efficient to use $[-1,1]^n$ and it is hard to compute $\mathbb{P}_i$.
Since $\mathbb{P}_i \subseteq \mathbb{P}$, we know once $x$ happens to fall into $\{\bar{x}_i\} \oplus G_i \mathbb{P}_i$, then the system is guaranteed to reach the target region.
%Second, in practice, for example for the half-cylinder flipping experiment, it almost never happens that $x\in \{\bar{x}_i\} \oplus G_i \mathbb{P}$, and hence $x\notin \{\bar{x}_i\} \oplus G_i \mathbb{P}_i$.

\iffalse 
In our hardware experiment, in order to simplify the linearized piecewise affine system, we assume there is no sliding.
We only distinguish between two objects in contact or not in contact. 
When in contact, the two objects are sticking and not sliding.
This assumption reduces the dimension of the states, and simplifies the linearized PWA system, making the system easier to control.
In practice this simplification works better.
\fi

\section{Experiment}

\label{sec:experiment}
\begin{figure*}
    \includegraphics[width=0.12\linewidth]{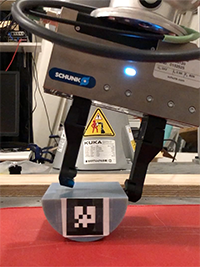}
    \includegraphics[width=0.12\linewidth]{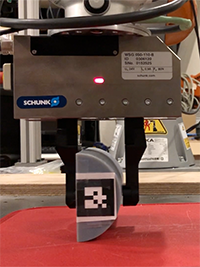}
    \includegraphics[width=0.12\linewidth]{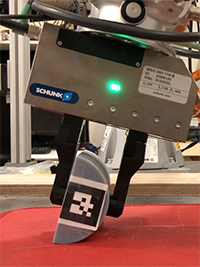}
    \includegraphics[width=0.12\linewidth]{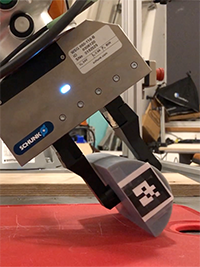}
    \includegraphics[width=0.12\linewidth]{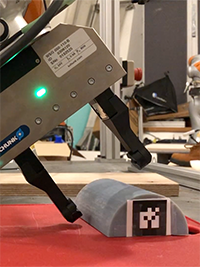}
    \includegraphics[width=0.12\linewidth]{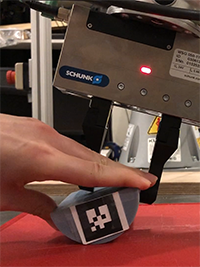}
    \includegraphics[width=0.12\linewidth]{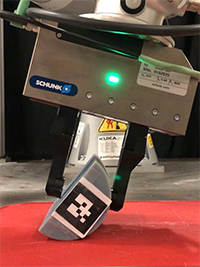}
    \includegraphics[width=0.12\linewidth]{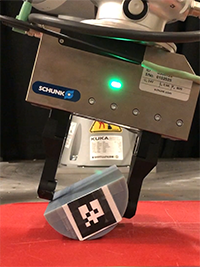}
    
    \quad\quad\ (A) \quad\ \  
    \quad\quad\quad (B) \quad\  \  
    \quad\quad\quad (C) \quad\  \  
    \quad\quad\quad (D) \quad\  \  
    \quad\quad\quad (E) \quad\  \  
    \quad\quad\quad (F) \quad\  \  
    \quad\quad\quad (G) \quad\  \  
    \quad\quad\quad (H) \quad
    \caption{Flipping the half-cylinder}
    \label{fig:flipping_motions}
\end{figure*}

We carried out the experiment on a Kuka robot with a two-finger Schunk gripper. 
The half-cylinder representing the carrot is 0.11 m long and its radius is 0.036 m.
We put an AR-tag \cite{fiala2005artag} with side width 3 cm on a side of the half-cylinder and use a Kinect to track the pose of the half-cylinder.
We can get in real time the pose of the gripper relative to the Kuka base and hence the pose of the gripper relative to the table. 
Once we compute the initial pose of the half-cylinder relative to the table, we can track the pose of the half-cylinder relative to the gripper in real time.
{We determine whether the gripper is in contact of the half-cylinder or not by relative poses, since we do not have any force or tactile sensors.}

The goal is to flip the half-cylinder 180 degrees, i.e., to manipulate the half-cylinder with the flat surface facing upwards to the pose where the flat surface is facing downwards to the table. 
We use our algorithm to design a trajectory that flips the half-cylinder 90 degrees so that the grippers are holding the half-cylinder.
After that, a manually-designed (open-loop) controller would transport the half-cylinder and flip another 90 degrees.
We mainly focus on the first 90-degree rotation for several reasons. First, it is most challenging in the entire manipulation process, and manually designed open-loop controllers usually fail in this phase.
Even an experienced human operator tele-operating the robot cannot accomplish this task in one or two attempts and the failures are often in the first 90-degree phase (see the accompanying video).
Second, in the next 90-degree rotation, the dynamics is different from that in the first 90-degree rotation, so one needs to design a different trajectory separately, instead of simply extending the first trajectory.
Besides, it is very simple to design a good open-loop controller for the next 90-degree rotation (see the accompanying video).

The state is $\textbf{x} = [x,y,\theta, \dot{x},\dot{y},\dot{\theta}, \varphi, w]$, where $x$ and $y$ are the coordinates of CoM of the half-cylinder, $\theta$ is the angle between the flat surface of the half-cylinder and the horizontal axis, $\dot{x},\dot{y},\dot{\theta}$ are the first-order time derivatives of $x,y,\theta$, respectively, $\varphi$ is the angle between the finger of the gripper and the horizontal axis, and $w$ is the separation between two fingers.
The control input is $\textbf{u} = [F_N,F_t, F_1,F_{1t},F_2,F_{2t}, \dot{\varphi}, \dot{w}]$, where $F_N,F_t$ are the normal force and the friction between the half-cylinder and the ground, and similar definitions for $F_1, F_{1t}, F_2, F_{2t}$ (Figure \ref{fig:contact_modes}).
We set up trajectory optimization with time horizon $N = 100$ and sampling time $dt = 0.01$ s.
%We constrain that the contact point between the left finger and the half-cylinder does not change over the entire trajectory.
{We constrain that there is always contact between the left finger and the flat surface of the half-cylinder. The friction coefficients are chosen to be between 0.2 and 0.5.}
The goal state region of the trajectory optimization is $X_G = \{\tbx: \varphi = \theta = 90^\circ \}$.
{The objective is to minimize $\sum_{t=0}^N |\varphi_t - \theta_t|$.}
It took IPOPT about 2 seconds to solve the trajectory optimization on an Intel i7 3.3
GHz, 32 GB RAM machine\footnote{By varying the time horizon or other constraints, the solving time of IPOPT varies from 1 to 25 seconds or IPOPT cannot find a feasible solution. We found that on solving this particular problem with varying constraints, IPOPT generally behaves better than SNOPT \cite{gill2005snopt}.}.

\begin{figure} 
    \centering
    \includegraphics[width=0.23\textwidth]{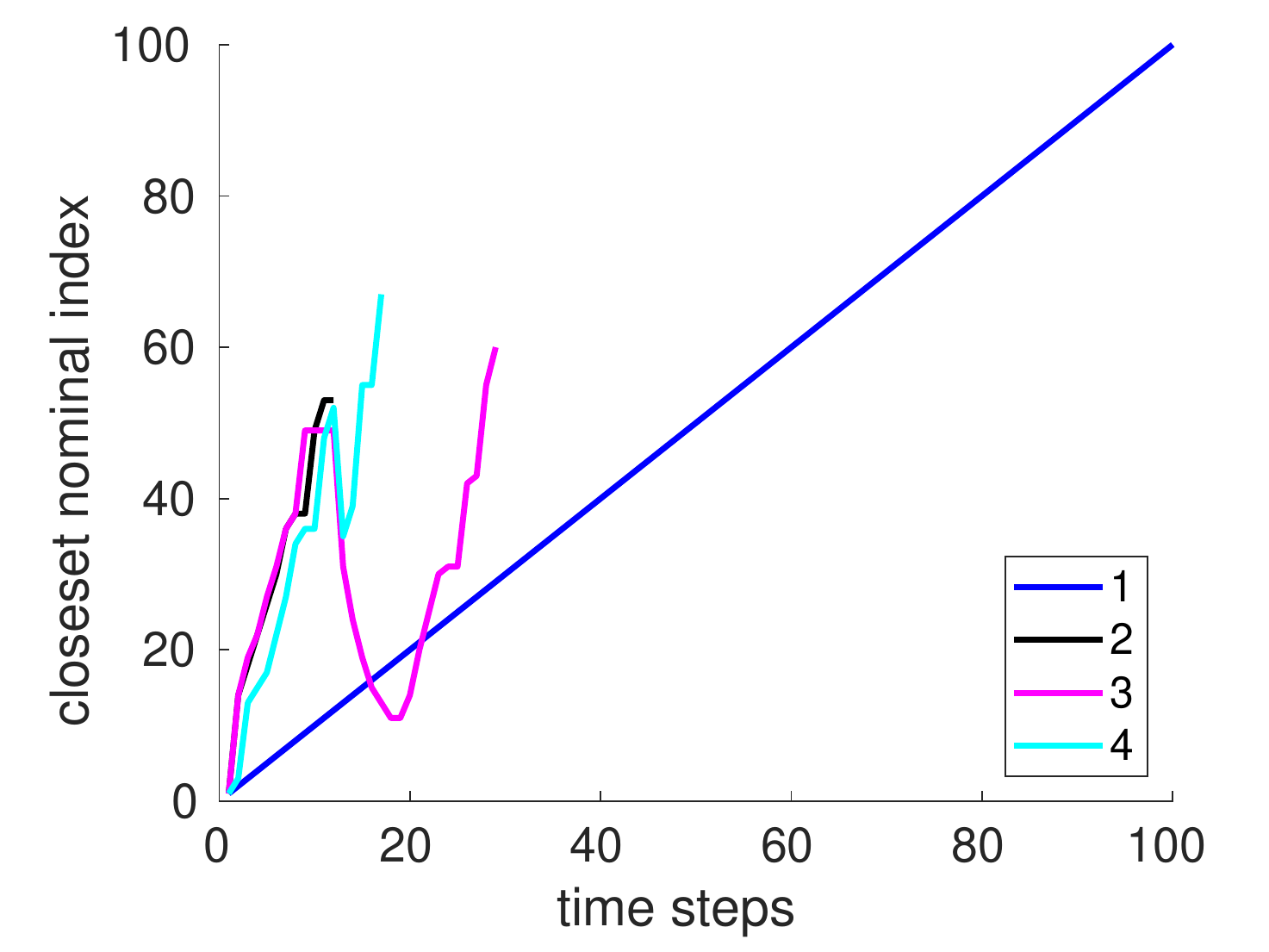}
    \includegraphics[width=0.23\textwidth]{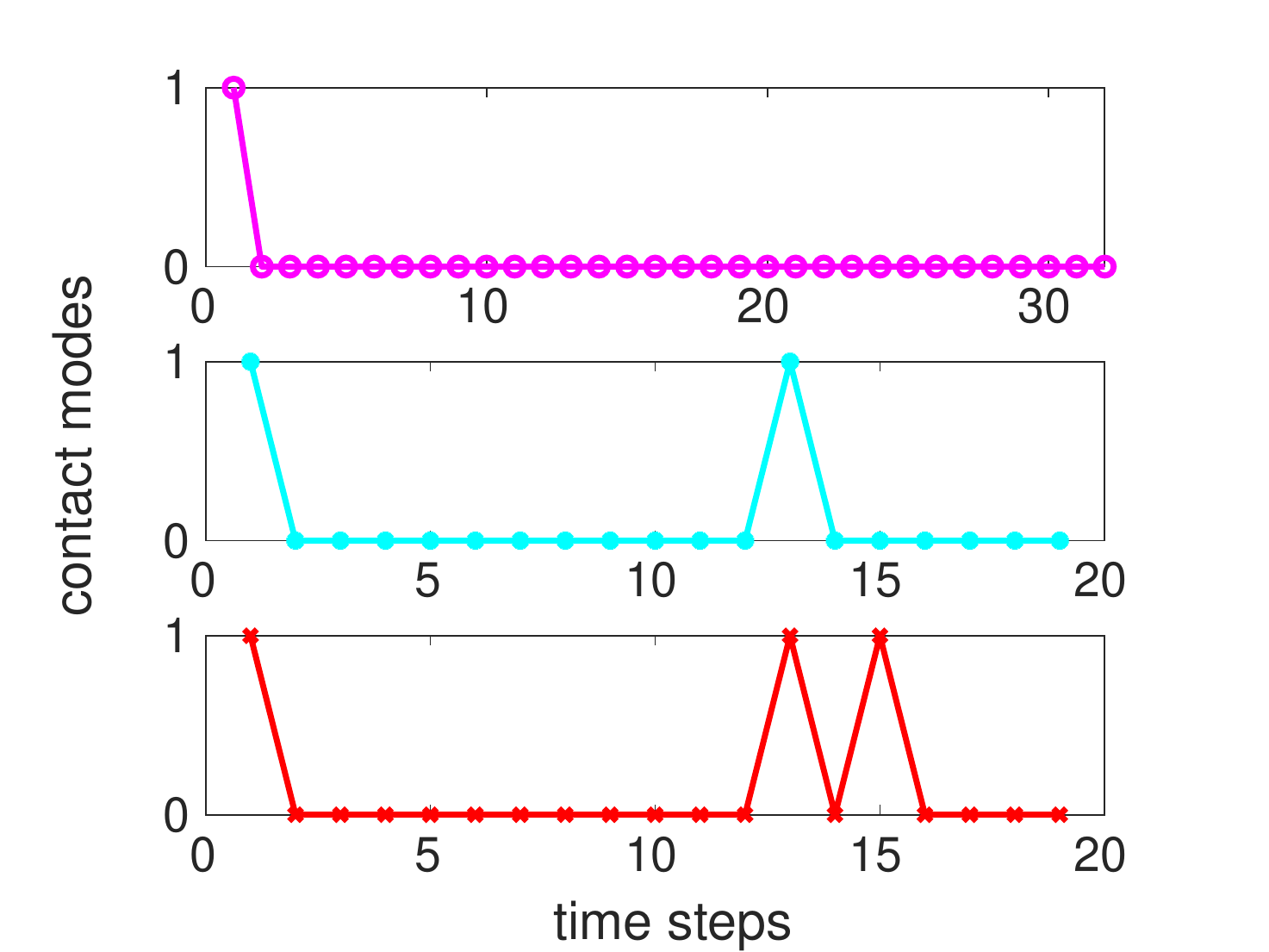}
    \caption{The figure on the left plots the index of the closest nominal trajectory state for the current state \textit{before} reaching the goal region $\widetilde{X}_G$ vs. time steps for four typical trajectories. Trajectory 1 can be thought of as the open loop trajectory and is just for reference. Trajectory 2, 3, and 4 are typical trajectories for the closed-loop controller, the closed-loop controller under the first type of disturbance, and the closed-loop controller under the second type of disturbance, respectively. They terminate early, because $\widetilde{X}_G$ is strictly larger than ${X}_G$. The figure on the right plots the contact modes for 3 typical trajectories. From top to bottom are typical contact modes for the closed-loop controller under the first type of disturbance, the closed-loop controller under the second type of disturbance, and the close-loop controller under the second type of disturbance where the gripper opens 4 cm.}
    \label{fig:two_plots}
\end{figure}

We design controllers in 2D, and in execution the half-cylinder is 3D, so we always operate at the center section of the half-cylinder.
During execution, we let the goal region be $\widetilde{X}_G = \{\tbx: \varphi = \theta, w \leq c \} \supseteq X_G$ for some constant $c$. Once the state is in the goal region, the gripper is in grasp of the half-cylinder, ready for the remaining operations. 
The open-loop controller obtained from trajectory optimization brings the half-cylinder from 0 degree (Fig \ref{fig:flipping_motions}.A) to 90 degrees (Fig \ref{fig:flipping_motions}.B).
Then a manually-designed controller flips the half-cylinder completely (Fig \ref{fig:flipping_motions}.DE).
The closed-loop controller stops early once the goal region has been reached (Fig \ref{fig:flipping_motions}.C), followed by the same manually-designed controller.

The open-loop controller is already very robust during execution, so we tested some larger disturbances to show the robustness of the closed-loop controller.
We experimented two types of disturbances. One is to force the half-cylinder to rotate clockwise using a human hand (Fig \ref{fig:flipping_motions}.F). 
We can easily fail the open-loop controller by holding the half-cylinder long enough so that the controller finishes open-loop execution.
We found that the closed-loop controller always recovers when the disturbances are no more than 30 degrees.
The success rate is 100\% (20 out of 20). %\footnote{If the manually-designed controller in the second phase fails, which happens rarely, we do not count it. Same for the second type of disturbance.}.
We observe that the contact modes do not change under small disturbances (Fig \ref{fig:two_plots} Right).
We also observe that the controller can sometimes recover from very large disturbances that violate the assumptions we made when designing the trajectory, e.g., the disturbance is more than 30 degrees and the right finger is on the flat surface of the half-cylinder.
%For larger disturbances,  is likely to recover if the assumption that the left contact point should remain unchanged is not violated.
%Large violation of the assumption or loss of pose tracking are two main reasons for failures.
%Since it is not well-defined how large the disturbance can be, we do not report success rate for this case.
%When it fails, it was either caused by .
The other type of disturbance is to ``accidentally" open the gripper for 2 cm at a certain time step (before opening: Fig \ref{fig:flipping_motions}.G, after opening: Fig \ref{fig:flipping_motions}.H).
This tests the local multi-contact stabilizing controller.
The success rate is also 100\% (20 out of 20).
The contact modes change when the disturbance happens (Fig \ref{fig:two_plots} Right).
The controller can also recover from large disturbances that violate our design assumptions, e.g., the gripper opens 4 cm at a certain time step, {in which case the left finger moves a long distance along the flat surface of the half-cylinder, or even breaks contact with the half-cylinder.}

\begin{table}
\centering 
\begin{tabular}{|c|c|c|c|}%{ |p{1.5cm}|p{1.5cm}|p{1.5cm}|p{1.5cm}|  }
 \hline
 \multicolumn{2}{|c|}{Disturbance 1} & \multicolumn{2}{|c|}{Disturbance 2}\\
 \hline
 $\approx 15^\circ$ & $\approx 30^\circ$ & at time step 7 & at time step 14\\
 \hline
 10/10 & 10/10 & 10/10 & 10/10\\
 \hline 
\end{tabular}
\caption{Success rate for recovery from two types of disturbances.}
\end{table}

\section{Discussion and Conclusion} 
\label{sec:conclusion}
We have described a locally robust feedback design algorithm for dexterous manipulation. 
We have shown on hardware that the algorithm can recover from large external disturbances.

There are some limitations. First, if the size of the half-cylinder or the object shape changes, we need to {rerun the whole pipeline:} analyzing the dynamics of the system, designing and solving the trajectory optimization, and building the stabilizing controllers offline. Building a large library for many shapes and various sizes is a possible solution to robust manipulation.
%Or machine learning might play a role. 
Second, we used AR-tag to track the pose of the object. The estimation for the contact points between the gripper and the object are very rough, which causes some problems sometimes. In practice, perception is important for manipulation.
Third, the algorithm does not work for some types of manipulation, for example throwing objects or sorting a deck of cards.
Enabling robots to do more complicated tasks using model-based approaches still remains to be explored.


%%%%%%%%%%%%%%%%%%%%%%%%%%%%%%%%%%%%%%%%%%%%%%%%%%%%%%%%%%%%%%%%%%%%%%%%%%%%%%%%

%%%%%%%%%%%%%%%%%%%%%%%%%%%%%%%%%%%%%%%%%%%%%%%%%%%%%%%%%%%%%%%%%%%%%%%%%%%%%%%%

%%%%%%%%%%%%%%%%%%%%%%%%%%%%%%%%%%%%%%%%%%%%%%%%%%%%%%%%%%%%%%%%%%%%%%%%%%%%%%%%
\iffalse
    \section*{APPENDIX}
\fi

\iffalse 
\section*{ACKNOWLEDGMENT}
This work was supported by Air Force/Lincoln Laboratory Award No. PO\# 7000374874, and Lockheed Martin Corporation Award No. RPP2016-002.
%We thank Twan Koolen, Sadra Sadraddini, Tao Pang, Greg Izatt, and Pete Florence for setting up hardware and giving valuable suggestions.
%Any opinions, findings, and conclusions or recommendations expressed in this material are those of the authors and do not necessarily reflect the views of our sponsors.
%The first author would like to thank Twan Koolen for helpful discussions about dynamics and humanoid robots, Sadra Sadraddini for helpful discussions on polytopic trees, Tao Pang for setting up the Kuka robot, Greg Izatt for providing initial code for controlling the Kuka robot, and Pete Florence for helping with the AR-tag and the Kinect and for tele-operating the Kuka robot.

\fi 
%%%%%%%%%%%%%%%%%%%%%%%%%%%%%%%%%%%%%%%%%%%%%%%%%%%%%%%%%%%%%%%%%%%%%%%%%%%%%%%%

\bibliographystyle{IEEEtran}
\bibliography{references}

\end{document}